\documentclass[a4paper, 11pt]{article}  
\usepackage{amsmath,amssymb,amsfonts,amsthm}
\usepackage{algorithm}
\usepackage[noend]{algorithmic}
\usepackage{graphicx}
\usepackage{comment}
\usepackage{xcolor}
\usepackage{xspace}
\usepackage{bbold}
\usepackage{tikz,pgfplots}
\pgfplotsset{compat=newest}
\usepackage{caption}
\usepackage{subcaption}
\usepackage{wrapfig}
\usepackage{ifthen}
\usepackage{url}
\usepackage{hyperref}
\usepackage{fullpage}

\usepackage[square,numbers]{natbib}
\newcommand\as{a.s.\xspace}
\newcommand\alg{BTT\xspace}
\newcommand{\0}{\mathbb{0}}
\newcommand{\1}{\mathbb{1}}
\newcommand{\im}{\textup{Im}}

\def\frechet{Fr\'echet\xspace}

 \def\C{\mathcal{C}} 
  \def\L{\mathcal{L}}
  
 \def\G{\mathcal{G}} 
 \def\T{\mathcal{T}} \def\L{\mathcal{L}}
  \def\B{\mathcal{B}}
  
\def\M{\mathcal{M}} \def\X{\mathcal{X}} 
 \def\M{\mathcal{M}}

 \def\dN{\mathbb{N}} 
\def\dR{\mathbb{R}}

\def\eps{\varepsilon}

\newcommand{\Cpp}{C\raise.08ex\hbox{\tt ++}\xspace}

\newtheorem{theorem}{Theorem}
\newtheorem{corollary}{Corollary}
\newtheorem{lemma}{Lemma}
\newtheorem{definition}{Definition}

\def\switcherarg{arxiv}

\newcommand{\switcher}[2] 
{ 
  \ifthenelse{\equal{\switcherarg}{icra}}
  {\textbf{$^\textup{ICRA}$} #1}
  {#2}
}

\title{Efficient Sampling-Based Bottleneck Pathfinding over Cost
  Maps}

\author{Kiril Solovey$^*$ and Dan Halperin\thanks{Kiril
    Solovey and Dan Halperin are with the Blavatnik School of Computer
    Science, Tel Aviv University, Israel.  Email:
    \texttt{\{kirilsol,danha\}@post.tau.ac.il}. This work has
    been supported in part by the Israel Science Foundation (grant
    no.~1102/11), by the Blavatnik Computer Science Research Fund, and
    by the Hermann Minkowski--Minerva Center for Geometry at Tel Aviv
    University. Kiril Solovey is also supported by the Clore Israel
    Foundation.} }

\begin{document}

\maketitle
\pagestyle{empty}

\begin{abstract}
  We introduce a simple yet effective sampling-based planner that is
  tailored for \emph{bottleneck pathfinding}: Given an
  implicitly-defined cost map $\M:\dR^d\rightarrow \dR$, which assigns
  to every point in space a real value, we wish to find a path
  connecting two given points, which minimizes the maximal value with
  respect to~$\M$. We demonstrate the capabilities of our algorithm,
  which we call \emph{bottleneck tree} (\alg), on several challenging
  instances of the problem involving multiple agents, where it
  outperforms the state-of-the-art cost-map planning technique
  T-RRT*. In addition to its efficiency, \alg requires the tuning of
  only a single parameter: the number of samples.  On the theoretical
  side, we study the asymptotic properties of our method and consider
  the special setting where the computed trajectories must be monotone
  in all coordinates. This constraint arises in cases where the
  problem involves the coordination of multiple agents that are
  restricted to forward motions along predefined paths.
\end{abstract}                  


\section{Introduction}
Motion planning is a widely studied problem in robotics. In its basic
form it is concerned with moving a robot between two given
configurations while avoiding collisions with obstacles. Typically, we
are interested in paths of high quality, which lead to lower energy
consumption, shorter execution time, or safer execution for the robot
and its surrounding. One of the most studied quality measures is 
path length, which was first studied from the combinatorial and
geometric perspective~\cite{Mitchell04} and has recently gained
popularity with the introduction of PRM*, RRT*~\cite{KarFra11} and subsequent
work. Another popular measure is the clearance (see,
e.g,~\cite{WeiETAL07}) of a path: the \emph{clearance} of a particular
robot configuration is the distance to the closest forbidden
configuration. The goal here is to find a path that maximizes the minimum
clearance of the configurations along it. 

The latter example is a special case of the \emph{bottleneck
  pathfinding} problem: we are given an implicitly-defined cost map
$\M:\C\rightarrow \dR$ which assigns to every configuration $c\in \C$
of the robot a value; the goal is to find, given start and target
configurations, a continuous path $\nu:[0,1]\rightarrow \C$, which
minimizes the expression $\max_{\tau\in [0,1]}\{\M(\nu(\tau))\}$.

Bottleneck pathfinding can also arise in settings involving multiple
robots. However, in most cases multi-robot motion planning is
computationally intractable (see,
e.g.,~\cite{hss-cmpmio,SolHal15}), even
when one is only concerned with finding \emph{a} solution. Thus,
\emph{decoupled} planners (see,
e.g.,~\cite{vanBerOve05,SimETAL02,AltETAL16}) usually construct a set
of $d$ paths---one for each robot---and attempt to coordinate between
the robots along their predefined paths in order to avoid
collisions. Now suppose that we also wish to find the \emph{safest}
coordination---the one which keeps the robots in a maximal distance
apart. This problem again, can be reformulated as a
bottleneck-pathfinding problem, where $d$ is the dimension of the
effective search space. In particular, denote by
$\sigma_i:[0,1]\rightarrow \C$ the path assigned for robot $i$, where
$1\leq i\leq d$, and define for a given $d$-dimensional point
$x=(x_1,\ldots,x_m)\in [0,1]^d$ the cost map
$\M(x):=1/\max_{1\leq i<j\leq c}\|\sigma(x_i)-\sigma(x_j)\|$. In other
cases, the robots are restricted a priori to predefined paths, as in
factory assembly lines~\cite{SpeETAL13}, aviation
routes~\cite{CafDur14} and traffic intersections~\cite{DasMou15}.

An interesting complication that typically arises from the examples of
the previous paragraph, is that each robot must move forward along its
path and cannot backtrack. This restriction induces paths that are
\emph{monotone} in each of the coordinates in the search space
$[0,1]^d$.  This is reminiscent of the notion of \emph{\frechet
  matching}, which can be viewed as another instance of bottleneck
pathfinding. \frechet matching is a popular similarity measure between
curves which has been extensively studied by the
computational-geometry community (see, e.g.,~\cite{AltGod95,HarRai14})
and recently was used in
robotics~\cite{VosETAL15,PokETAL16,HolSri16}. This matching takes into
consideration not only the ``shape'' of curves but also the order in
which the points are arranged along them. \vspace{3pt}

\noindent\textbf{Contribution.}
Rather than tackling each problem individually, we take a
sampling-based approach and develop an effective planner that can cope
with complex instances of the bottleneck-pathfinding problem.  We
demonstrate the capabilities of our technique termed \emph{bottleneck
  tree} (\alg) on a number of challenging problems, where it
outperforms the state-of-the-art technique T-RRT* by several orders of
magnitude. In addition to its efficiency, \alg requires the tuning of
only a single parameter: the number of samples. On the theoretical
side, we study the asymptotic properties of our method and consider
the special setting where the returned paths must be monotone in all
coordinates.

Notice that the monotonicity requirement makes the problem harder,
particularly in the sampling-based setting. Consider for example a PRM
$G$ that is constructed using a connection radius $r_n$ and $n$
samples. By removing all the non-monotone edges from $G$ we obtain a
much sparser graph $G'$: informally, $G'$ may retain only a $1/2^d$
fraction of the edges of $G$, where $d$ is the dimension of the
problem. Fortunately, our analysis of \alg indicates that in order to
guarantee optimality in the monotone regime, the connection radius
does not have to be drastically increased over its value in the
non-monotone case of say PRM.

We have already considered sampling-based bottleneck pathfinding in a
previous work~\cite{SolHal16}. However, there we employed a PRM-based
approach which is far inferior to \alg. Already when working in a
four-dimensional C-space, the PRM-base approach, which seeks to
explore the {\it entire} C-space, becomes prohibitively expensive. In
contrast, \alg quickly produces low-cost solutions even in a
seven-dimensional
space. 
We also provide here stronger theoretical analysis of our new
technique for the monotone case. These differences are discussed in
more detail in Sections~\ref{sec:theory} and~\ref{sec:experiments}.

In Section~\ref{sec:related} we review related work. In
Section~\ref{sec:preliminaries} we provide a formal definition of
bottleneck pathfinding. In Section~\ref{sec:algorithm} we describe our
\alg algorithm for the problem. In Section~\ref{sec:theory} we provide
an asymptotic analysis of the method in the monotone case. In
Section~\ref{sec:experiments} we report on experimental results and
conclude with a discussion and future work in
Section~\ref{sec:discussion}.

\section{Related work}\label{sec:related}
Particular instances of the basic motion-planning problem involving a
small number of degrees of freedom can be solved efficiently in a
complete manner~\cite{Sharir04}. Complete planners can even cope with
scenarios involving multiple robots, if some assumptions are made
about the input (see,
e.g.,~\cite{SolYuZamHal15,abhs-unlabeled14,tmk-cap13}. However, in
general motion planning is computationally
intractable~\cite{hss-cmpmio,SolHal15,sy-snp84,Rei79}. Thus, most of
the recent efforts in this area are aimed at the development of
sampling-based planners, which attempt to capture the connectivity of
the free space using random sampling.

We continue our literature review with general-purpose sampling-based
planners, with an emphasis on planners that are applicable to settings
involving cost maps. We then proceed to specific examples of
bottleneck pathfinding that were studied from the
combinatorial-algorithmic perspective.

\subsection{Sampling-based motion planning}
Early sampling-based planners such as PRM~\cite{kslo-prm} and
RRT~\cite{l-rert} have focused on finding \emph{a} solution. Although
some efforts were made to understand the quality of paths produced by
such planners~\cite{RavETAL11,NecETAL10,GerOve07} this issue remained
elusive until recently. In their influential work Karaman and
Frazzoli~\cite{KarFra11} introduced the planners PRM* and RRT*, which
were shown to be \emph{asymptotically optimal}, i.e., guaranteed to
return a solution whose cost converges to the optimum. Several
asymptotically-optimal planners that focus on the path-length cost
have later emerged~\cite{ArsTsi13,JanETAL15,SalHal14a}.  In our recent
work~\cite{SolETAL16} we develop a general framework for the analysis
of theoretical properties of a variety of sampling-based planners
through a novel connection with the theory of random geometric graphs.

We now proceed to discuss sampling-based planning in the presence of
an underlying cost map $\M$, as in bottleneck pathfinding. RRT*
theoretically guarantees to converge to the optimum even when the cost
of the path is computed with respect to~$\M$. However, in practice
this convergence is very slow, since RRT* does not take into account
the cost of $\M$ while exploring $\C$, and puts too much effort in the
traversal of regions of high cost. On the other hand,
T-RRT~\cite{JaiETAK10} equips RRT with a \emph{transition test}, which
biases the exploration towards low-cost regions of~$\M$. However,
T-RRT does not take path quality into consideration and cannot improve
the cost of a solution that was already obtained. Consequently, this
planner provides no optimality guarantees.  The T-RRT*
algorithm~\cite{DevETAL16} combines RRT* with T-RRT, which yields an
asymptotically-optimal planner that also works well in practice.  The
authors of~\cite{DevETAL16} test T-RRT* for the case of \emph{integral
  cost}, which sums the cost of the configurations along the path, and
\emph{mechanical cost}, which sums the positive cost variations along
the path. We mention that T-RRT* is also applicable to the bottleneck
cost, although this setting was not tested in that work. Lastly, we
note that we are not aware of other techniques that consider planning
in the presence of a cost map in the general setting of sampling-based
planning.

\subsection{Bottleneck pathfinding and other problems} 
Cost maps typically arise as \emph{implicit} underlying structures in
many problems involving motion and path planning. de Berg and van
Kreveld~\cite{BerKre97} consider the problem of path planning over a
\emph{known} mountainous region. They describe data structures that
can efficiently obtain, given two query points (which are the start
and goal of the desired path), paths with various properties, such as
those that strictly descend and paths that minimize the maximal height
obtained, i.e., bottleneck paths.

We have already mentioned the problem of high-clearance paths. For the
simple case of a disc robot moving amid polygonal obstacles this
problem can be solved efficiently by constructing a Voronoi diagram
over the sites that consist of the obstacles~\cite{ODuYap85}. There is
an obvious trade-off between clearance and path length, a short paths
tend to get very close to the obstacles. The work by Wein et
al.~\cite{WeiETAL07} describes a data structure that given a clearance
threshold $\delta>0$ returns the shortest path with that
clearance. Agarwal et al.~\cite{AgaETAL16} consider a particular cost
function suggested by Wein et al.~\cite{WeiETAL08} that combines path
length and clearance and develop an efficient approximation algorithm
for this case.

Perhaps the most popular example of bottleneck pathfinding from the
recent years is \emph{\frechet matching}, which is concerned with
quantifying the similarity of two given curves: the \frechet distance
between two given curves can be described as the shortest leash that
allows a person to walk her dog, where the former is restricted to
move along the first curve and latter along the other. It is generally
assumed that this quantity tends to be more informative when only
\emph{forward motions} along the curves are permitted, which transform
into \emph{monotone} paths in the search space. \frechet matching has
been extensively studied for the case of two curves: several efficient
techniques that solve the problem exist (see,
e.g.,~\cite{AltGod95,BucBucMeuMul14}). The problem can be extended to
multiple curves although only algorithms that are exponential in the
number of curves are known~\cite{HarRai14,BucETAL16}. Several papers
consider extensions of the problem involving more complex input
objects~\cite{BucBucWen08,BucBucSch08} or cases where additional
constraints such as obstacles are imposed on the
``leash''~\cite{CooWen10,ChaETAL10}. Finally, we mention our recent
work~\cite{SolHal16} where we apply a sampling-based PRM-like
technique for solving several \frechet-type problems.

\section{Preliminaries}\label{sec:preliminaries}
We provide a formal definition of the bottleneck-pathfinding problem.
For some fixed dimension $d\geq 2$, let $\M:[0,1]^d\rightarrow \dR$ be
a cost map that assigns to every point in $[0,1]^d$ a value in $\dR$.
We will refer from now on to $[0,1]^d$ as the \emph{parameter space}
in order to distinguish it from the configuration space of the
underlying problem.

Notice that the definition of $\M$ does not imply that we restrict our
discussion to the setting of a single robot whose configuration space
is $[0,1]^d$. For instance, $[0,1]^d$ can represent the parameter
space of $d$ curves of the form $\sigma_i:[0,1]\rightarrow \C_i$,
where for every $1\leq i\leq d$, $\sigma_i$ describes that motion of
robot $i$ in the configuration space $\C_i$. Typically in such
settings involving multiple robots having predefined paths the robots
are only permitted to move forward along their paths and never
backtrack. This restriction induces \emph{monotone} motions in
$[0,1]^d$. Given two points
$x=(x_1,\ldots,x_d),y=(y_1,\ldots,y_d)\in \dR^d$ the notation
$x\preceq y$ implies that for every $1\leq i\leq d$ it hold that
$x_i\leq y_i$. Additionally, we use the notation $x\preceq_{\delta}y$
for $\delta>0$ to indicate that $x\preceq y$ and for every
$1\leq i\leq d$ it hold that $y_i - x_i \geq \delta$.

For simplicity we will assume that the goal consists of planning paths
between the points $\0,\1\in [0,1]^d$, where
$\0=(0,\ldots,0), \1=(1,\ldots,1)$.

\begin{definition}
  A plan $\nu:[0,1]\rightarrow [0,1]^d$ is a continuous path in
  $[0,1]^d$ that satisfies the following two constraints:
  (i)~$\nu(0)=\0,\nu(1)=\1$; (ii)~it is monotone in each of the $d$
  coordinates, i.e., for every $0\leq \tau \leq \tau'\leq 1$ and
  $\nu(\tau)\preceq \nu(\tau')$.
\end{definition}

Given a path (or a plan) $\nu$, its \emph{bottleneck cost} is defined
to be $\M(\nu)=\max_{\tau\in [0,1]}\M(\nu(\tau))$. In
Section~\ref{sec:experiments} we will address specific examples of the
following problem:

\begin{definition}
  Given a cost map $\M:[0,1]^d\rightarrow \dR$ the (monotone)
  bottleneck-pathfinding problem consists of finding a plan $\nu$
  which minimizes $\M(\nu)$.
\end{definition}

\section{Bottleneck-tree planner}\label{sec:algorithm}

\begin{algorithm}\caption{\textsc{bottleneck-tree}$(n,r_n)$\label{alg:btt}}
  \begin{algorithmic}[1]
    \STATE $V := \{\0,\1\} \cup \texttt{sample}(n)$;
    \FORALL {$x\in V$}
    \STATE $c(x):=\infty$; $p(x):=\textup{null}$
    \ENDFOR
    \STATE $c(\0):= \M(\0)$
    \WHILE {$\exists x\in V\textup{ such that }c(x)<\infty$} 
    \STATE $z:= \texttt{get\_min}(V)$
    \IF {$z==\1$} 
    \RETURN $\texttt{path}(T,\0,\1)$ 
    \ENDIF
    \STATE $V := V \setminus \{z\}$ 
    \STATE $N_z:= \texttt{near}(z,V,r_n)$
    \FORALL {$x\in N_z$} 
    \IF {$z\preceq x$ \textbf{and} $c(z)<c(x)$}
    \STATE $c_{\textup{new}}:= \max\{c(z),\M(z,x)\}$ 
    \IF {$c_{\textup{new}}<c(x)$}
    \STATE $c(x):=c_{\textup{new}}$; $p(x):=z$
    \ENDIF
    \ENDIF
    \ENDFOR
    \ENDWHILE
    \RETURN $\emptyset$
  \end{algorithmic} 
\end{algorithm}

In Algorithm~\ref{alg:btt} we describe our sampling-based
technique for bottleneck pathfinding, which we call bottleneck tree
(\alg). The algorithm can be viewed as a lazy version of PRM which
explores the underlying graph in a Dijkstra-like fashion, according to
cost-to-come, with respect to the bottleneck cost over $\M$. Thus, it
bears some resemblance to FMT*~\cite{JanETAL15} since the two
implicitly explore an underlying PRM. However, FMT* behaves very
differently from \alg since it minimizes the path-length cost.

\alg accepts the number of samples $n$
and the connection radius $r_n$ as parameters.  It maintains a
directed minimum spanning tree $\T$ of an underlying \emph{monotone}
PRM or a random geometric graph (RGG)~\cite{SolETAL16}\footnote{PRMs
  and RGGs are equivalent in our context.}:
\begin{definition}\label{def:monotone_rgg}
  Given $n\in \dN_+$ denote by $\X_n=\{X_1,\ldots,X_n\}$, $n$ points
  chosen independently and uniformly at random from $[0,1]^d$. For a
  connection radius $r_n>0$ denote by $\G_n=\G(\{\0,\1\}\cup\X_n;r_n)$
  the \emph{monotone} RGG defined over the vertex set
  $\{\0,\1\}\cup\X_n$. $\G_n$ has a directed edged $(x,y)$ for
  $x,y\in \{\0,\1\}\cup\X_n$ iff $x\preceq y$ and $\|x-y\|\leq r_n$,
  where $\|\cdot\|$ represents the standard Euclidean distance.
\end{definition}
The algorithm keeps track of the unvisited vertices of $\G_n$ using
the set $V$. It maintains for each vertex $x$ the cost-to-come over
the visited portion of $\G_n$ and its parent for this specific path,
are denoted by $c(x)$ and $p(x)$, respectively.

In line~1 we initialize $V$ with $\0,\1$, and a set of $n$ uniform
samples in $[0,1]^d$, which are generated using $\texttt{sample}$.  In
lines~2-3 the cost-to-come and the parent attributes are
initialized. In each iteration of the \textbf{while} loop, which
begins in line~5, the algorithm extracts the vertex $z\in V$ which
minimizes $c(z)$ , using \texttt{get\_min} (line~6). If~$z$ turns out
to be~$\1$ (line~7), then a path from $\0$ to $\1$ is returned. This
is achieved using \texttt{path}, which simply follows the ancestors of
$\1$ stored in the $p$ attribute until reaching~$\0$. Otherwise, $z$
is removed from $V$ (line~9) and a subset of its neighbors in $\G_n$
that are also members of $V$ are obtained using \texttt{near}
(line~10). For each such neighbor $x$ (line~11) it is checked whether
$z\preceq x$ and whether the cost-to-come of $x$ can potentially
improve by arriving to it from $z$ (line~12). An alternative
cost-to-come for $x$ which is induced by a path that reaches from $z$
is calculated (line~13). In particular, this cost is the maximum
between the cost-to-come of $z$ and the cost of the edge\footnote{The
  cost of an edge with respect to a given cost map can be approximated
  by dense sampling along the edge, as is customary in motion
  planning.} from $z$ to $x$ with respect to $\M$. If the alternative
option improves its cost then the parents and the cost of $x$ are
updated accordingly (line~15).

In preparation for the following section, where we study the
asymptotic behavior of \alg, we mention here that given that the
underlying graph~$\G_n$ contains a path from~$\0$ to~$\1$, \alg finds
the minimum bottleneck path over~$\G_n$, namely the path whose maximal
$\M$ value is minimized. We omit a formal proof of this fact here. It
is rather straightforward and is very similar similar to the
completeness proof of the standard Dijkstra algorithm, but for the
bottleneck cost; see, e.g.,~\cite[Theorem~25.10]{CorETAL09}.

\section{Asymptotic analysis}\label{sec:theory}
In the previous section we have argued that \alg returns a solution
that minimizes the bottleneck cost over a fixed graph $G\in \G_n$.  A
major question in the analysis of sampling-based motion-planning
algorithms is under what conditions does a discrete graph structure
captures the underlying continuous space. For this purpose, we study
the asymptotic properties of $\G_n=\G(\{\0,\1\}\cup\X_n;r_n)$
(Definition~\ref{def:monotone_rgg}). To the best of our knowledge,
existing proofs concerning asymptotic optimality or completeness of
sampling-based planners (see, e.g.,~\cite{KarFra11,JanETAL15}) only
apply to the non-monotone (and standard) setting. Such results cannot
be extended as-are to the analysis of $\G_n$. The remainder of this
section is dedicated to strengthening our our analysis concerning the
special and harder case that imposes the monotonicity requirement.

Previously~\cite{SolHal16}, we were able to give a convergence
guarantee in the monotone case by introducing a connection radius
($r_n$) of the form $f(n)\left(\frac{\log n}{n}\right)^{1/d}$, where
$f(n)$ is any function that diminishes as $n$ tends to
$\infty$. Unfortunately, this statement does not indicate which
specific function should be used in practice, and an uncareful
selection of $f$ could lead to undesired behavior of \alg. In
particular, if $f$ grows too slowly with $n$ the graph becomes
disconnected and \alg may not be able to find a solution at all. On
the other hand, $f$ that grows too fast with $n$ could result in a
needlessly dense graph which will take very long time to traverse.  In
this paper, we replace this somewhat abstract definition with a
connection radius of the form
$\gamma\left(\frac{\log n}{n}\right)^{1/d}$, where $\gamma$ is a
constant depending only on $d$. This is a more typical structure of
bounds in the context of random geometric graphs and sampling-based
planners. Furthermore, we show in Section~\ref{sec:experiments} that
this formulation of $r_n$ leads to a quick convergence of \alg to the
optimum.

For the purpose of the analysis, we define a supergraph of $\G_n$,
denoted by $\L_n=\L(\{\0,\1\}\cup\X_n;r_n)$, which corresponds to the
standard and non-monotone random geometric graph
(see~\cite{SolETAL16}). In particular, it connects every pair of
vertices $x,y\in\{\0,\1\}\cup\X_n$ with an edge if and only if
$\|x-y\|\leq r_n$, regardless of whether the monotonicity constraint
is satisfied.

In order to guarantee asymptotic optimality, one typically has to make
some assumptions with respect to the solution one hopes to converge
to. First, we introduce basic definitions.  Denote by $\B_{r}(x)$ the
$d$-dimensional Euclidean ball of radius $r>0$ centered at
$x\in \dR^d$ and $\B_{r}(\Gamma) = \bigcup_{x \in \Gamma}\B_{r}(x)$
for any $\Gamma \subseteq \dR^d$.  Given a curve
$\nu:[0,1]\rightarrow \dR^d$ define
$\B_r(\nu)=\bigcup_{\tau\in[0,1]}\B_r(\nu(\tau))$.  Additionally,
denote the image of a curve $\nu$ by
$\im(\nu)=\bigcup_{\tau\in[0,1]}\{\nu(\tau)\}$.

\begin{definition}\label{def:rgg}
  Given $\M:[0,1]^d\rightarrow \dR$, a plan
  $\nu:=[0,1]\rightarrow [0,1]^d$ is called \emph{robust} if for every
  \mbox{$\eps>0$} there exists \mbox{$\delta>0$} such that for any
  plan $\nu'$ for which \mbox{$\im(\nu)\subset \B_\delta(\nu)$} it
  follows that $\M(\nu')\leq (1+\eps)\M(\nu)$. A plan that attains the
  infimum cost, over all robust plans, is termed \emph{robustly
    optimal} and is denoted by~$\nu^*$.
\end{definition}

The following theorem is the main theoretical contribution of this
paper. The constant $\gamma$, which is defined below, was first
obtained in~\cite{JanETAL15}, for a different problem, and involving
non-monotone connections.

\begin{theorem}\label{thm:main_new}
  Suppose that $\L_n=\L(\X_n\cup\{\0,\1\};r_n)$ with
  $r_n=\gamma \left(\frac{\log n}{n}\right)^{1/d}$, where
  $\gamma =(1+\eta)2(d\theta_d)^{-1/d}$ with $\theta_d$ representing
  the Lebesgue measure of the unit Euclidean ball, and $\eta>0$ is any
  positive constant. Then $\L_n$ contains a path $\nu_n$ connecting
  $\0$ to $\1$ which has the following properties \as: \textbf{(i)}
  $\M(\nu)=(1+o(1))\M(\nu^*)$; \textbf{(ii)} at most $o(1)$ fraction
  of the edges along $\nu_n$ are non-monotone.
\end{theorem}

This theorem implies that $\L_n$ contains a path $\nu$ that is
\emph{almost} entirely monotone and its cost converges to the
optimum. Notice that in theory this does not necessarily mean that
\alg is asymptotically optimal, as the formulation of the algorithm in
Alg.~\ref{alg:btt} is concerned with paths that are \emph{entirely}
monotone. However, our experiments (Section~\ref{sec:experiments})
suggest that this theoretical gap can be bridged.

  \subsection{Proof of Theorem~\ref{thm:main_new}}
  This subsection is devoted to the proof of
  Theorem~\ref{thm:main_new}.  It relies on some components that were
  previously described in the works of Janson et al.~\cite{JanETAL15},
  Karaman and Frazzoli~\cite{KarFra11}, and our
  work~\cite{SolHal16}. However, we note that several ideas employed
  here are brand new and so is the final result.

  Our proof follows the common ``ball-covering''
  argument~\cite{KarFra11,JanETAL15,KavETAL96}, which shows that there
  exists a collection of samples from $\X_n$, that are located in the
  vicinity of ``landmark'' points that lie on some plan of
  interest. This guarantees that $\L_n$ contains a path that has a
  cost similar to the robustly optimal plan $\nu^*$. Then we use a
  refined ``$(\alpha,\beta)$ ball-covering''~\cite{JanETAL15}
  argument, which shows that many of the samples lie very closely to
  the landmarks---so close that most of them form monotone
  subchains. This induces the existence of paths that are
  almost-entirely monotone and have desirable bottleneck cost.

  We start by selecting a fixed $\eps>0$. As $\nu^*$ is robustly
  optimal, there exists $\delta>0$ such that for every plan $\nu$ with
  $\im(\nu)\subseteq \B_{\delta}(\nu^*)$ it follows that
  $\M(\nu)\leq (1+\eps)\M(\nu^*)$. Define $\ell=2/\delta$. Similarly
  to the proof of Theorem~6 in~\cite{SolHal16}, there exists a
  sequence of $\ell - 1 \leq k\leq \ell$ points
  $Q=\{q_1,\ldots, q_k\}\subset \im(\nu^*)$ where $q_1=\0,q_k=\1$, and
  $q_j\preceq_{\delta/2} q_{j+1}$ for any $1\leq j<k$. Note that
  $\ell$ is finite and independent of $n$.  This follows from the fact
  that $\0\preceq_1\1$ and every subpath of $\nu^*$ must ``make
  progress'' in relation to at least one of the coordinates. In
  particular, for every $1\leq j\leq d$ define
  $Q^j=\{q_1^j,\ldots,q_\ell^j\}\subset \im(\nu^*)$ to be a collection
  of $\ell$ points along $\nu^*$ with the following property: for
  every $1\leq i'\leq \ell$ the $j$th coordinate of $q_{i'}^j$ is
  equal to~$j\cdot 2/\delta$. Observe that for any $1\leq i'\leq \ell$
  there exists $1\leq j\leq d$ and $1\leq i''\leq \ell$ such that
  $\0\preceq_{(j-1)\cdot 2/\delta} q^j_{i''}$ and
  $\0\not \preceq_{(j-1)\cdot 2/\delta+\Delta} q^j_{i''}$ for any
  $\Delta>0$, i.e., one of the coordinates of $q^j_{i''}$ is equal to
  $j\cdot 2/\delta$ (see Figure~\ref{fig:thm_monotonicity}). Thus, it
  follows that $Q\subset \bigcup_{j=1}^d Q^j$.

\begin{figure}
  \centering
  \includegraphics[width=0.5\columnwidth, trim=0pt 0pt 0pt 0pt ,
  clip=true]{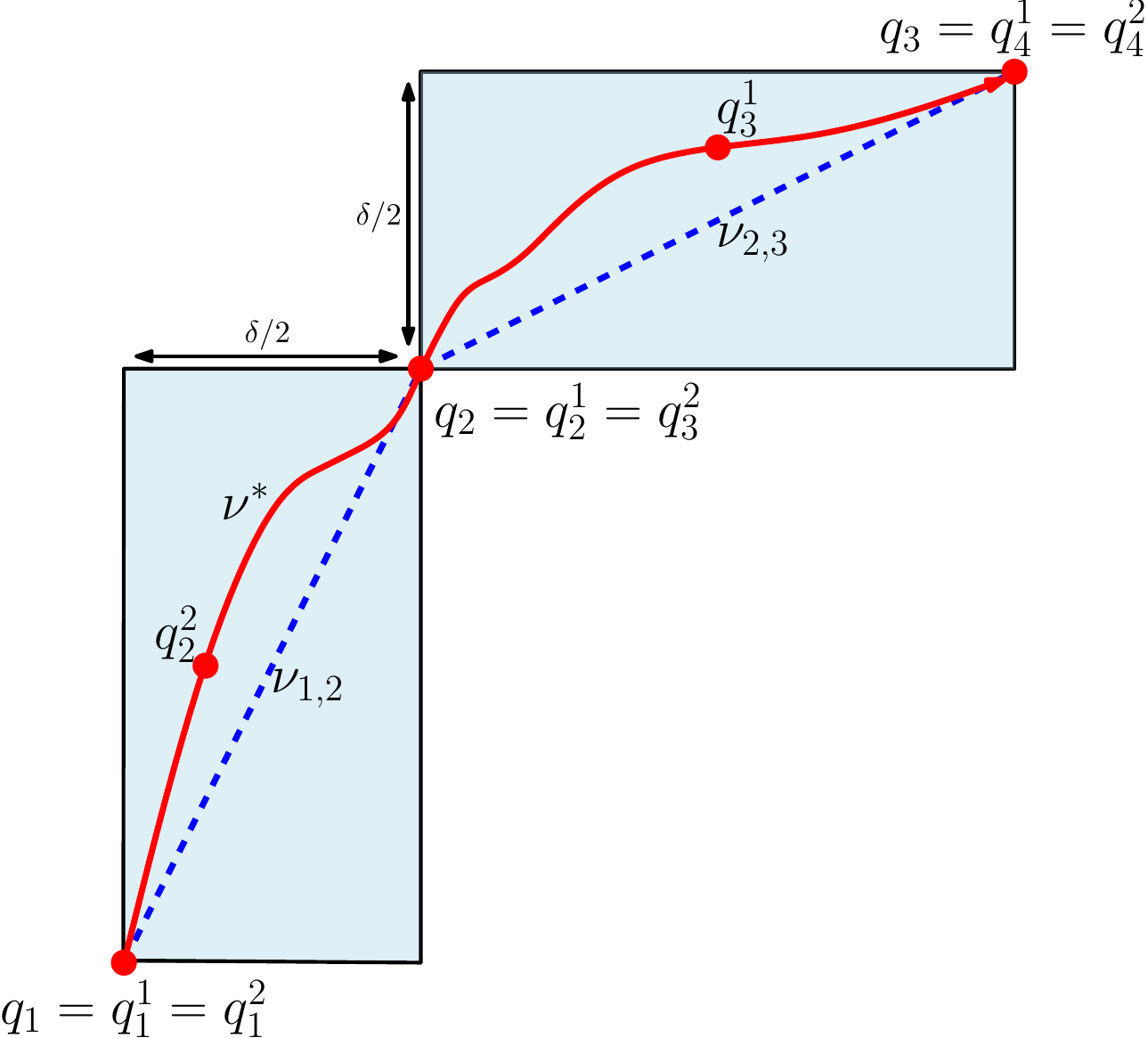}
  \caption{Visualization of the proof of Theorem~\ref{thm:main_new}
    for $d=2$. The red curve represents $\nu^*$, on
    which lie the points $Q^1\cup Q^2$ (represented by red
    bullets). Observe that $Q\subset Q^1\cup Q^2$. The dashed blue
    straight lines represent $\nu_{1,2},\nu_{2,3}$.
    \label{fig:thm_monotonicity}}
\end{figure}

As~$\nu^*$ can be arbitrarily complex we define a simplified plan,
which is located in the vicinity of $\nu^*$, and is consequently of
low cost. Firstly, for every $1\leq j<k$ denote by $\nu_{j,j+1}$ the
straight-line path from $q_j$ to $q_{j+1}$ (see
Figure~\ref{fig:thm_monotonicity}). Now, define $\nu_{1,k}$ to be the
concatenation of these $k-1$ subpaths. Observe that $\nu_{1,k}$ is a
plan.

Fix $\psi\in (0,1)$ and define $r'_n:=\tfrac{r_n}{2(2+\psi)}$. We
generate a set of $M_n$ ``landmark'' points
$P=\{p_1,\ldots, p_{_{\!_{M_n}}}\}$, which are placed along
$\nu_{1,k}$, i.e., $P\subset \im(\nu_{1,k})$, such that for every
$1\leq i\leq M_n-1$ it holds that $\|p_i-p_{i+1}\|\leq r'_n$. Note
that $p_1=\0$. For simplicity we also assume that $p_{_{\!_{M_n}}}=\1$.

Now define a set of balls centered at the landmarks. In particular,
for every $1\leq i\leq M_n$ define the set
$B_{n,i}:=\B_{r'_n}(p_i)$. Additionally, denote by $A_n$ the event
representing that for every $1\leq i\leq M_n$ it holds that
$\X_n\cap B_{n,i}\neq \emptyset$. We have the following lemma, which
is proven in~\cite[Lemma~4.4]{JanETAL15}:

\begin{lemma}\label{lem:pavone1}
  The event $A_n$ holds \as
\end{lemma}

\begin{corollary}\label{cor:quality}
  The graph $\L_n$ contains a path $\nu_n$ connecting $\0$ to $\1$
  such that $\M(\nu_n)\leq (1+\eps)\M(\nu^*)$ \as.
\end{corollary}
\begin{proof}
  Suppose that $A_n$ holds. Then for every $1<i<M_n$ there exists
  $x_i\in \X_n$ such that $x_i\in B_{n,i}$. Observe that for any
  $1\leq i <M_n$ it holds that $\|x_i-x_{i+1}\|\leq r_n$ and so $\L_n$
  contains an edge from $x_i$ to $x_{i+1}$. Thus, we can define
  $\nu_n$ to be the path consisting of the points
  $\0,x_2,\ldots, x_{M_n-1}, \1$.

  For any $1<i<M_n$ denote by $p'_i$ a point along $\nu^*$ which is
  the closest to $p_i$. Observe that $\|p_i-p'_i\|\leq \delta /2$.
  This, combined with the fact that $\|x_i-p_i\|\leq r'_n$, implies
  that $\|x_i-p'_i\|\leq \delta / 2 + r'_n\leq \delta$, which then
  guarantees that $\M(x_i)\leq (1+\eps)\M(\nu^*)$, by definition of
  $\nu^*$. Similar arguments can be applied to the edges of
  $\nu_n$. Thus, $\M(\nu_n)\leq (1+\eps)\M(\nu^*)$. 
\end{proof}

In order to reason about the monotonicity of $\nu_n$ we make the
following crucial observation:

\begin{lemma}
  There exists a fixed constant $\beta\in (0,1)$ such that for any
  $1\leq i<M_n$ it holds that $p_i\preceq_{\beta r_n} p_{i+1}$.
\end{lemma}

\begin{proof}
  First, assume that there exists $1\leq j<k$ for which
  $p_i,p_{i+1}\in \im(\nu_{j,j+1})$, and note that by definition the
  straight line from $p_i$ to $p_{i+1}$ is a subsegment
  of~$\nu_{j,j+1}$. Now, recall that $q_j\preceq_{\delta/2}q_{j+1}$
  and $\|p_i-p_{i+1}\|=r'_n$. Additionally, denote by
  $L_{j,j+1}=\|q_j-q_{j+1}\|$ and note that this value is a constant
  independent of $n$. Finally, let $\beta_j$ be the maximal value for
  which $p_i\preceq_{\beta_j r_n} p_{i+1}$ and observe that
  $\frac{L_{j,j+1}}{r'_n}=\frac{\delta/2}{\beta_j r_n}$. Thus,
  $\beta_j = \frac{\psi\delta}{2(2+\psi)}$.

  Now consider the case where
  $p_i\in \im(\nu_{j,j+1}), p_{i+1}\in \im(\nu_{j+1,j+2})$. Without
  loss of generality, assume that
  $\|p_i-q_{j+1}\|\geq \|p_{i+1}-q_{j+1}\|$. Thus,
  $p_i\preceq_{\beta_j r_n / 2} q_{j+1}$, which implies that
  $p_i\preceq_{\beta_j r_n/2} p_{i+1}$. Finally, define
  $\beta=\min_{1\leq j<k}\{\beta_j/2\}$. 
\end{proof}

\begin{corollary}\label{cor:monotone}
  Suppose that $p_i,p_{i+1}\in P$ for some $1\leq i\leq M_n$. Then for
  any two points $x_i,x_{i+1}\in \X_n$ such that
  $x_i\in \B_{\beta' r_n}(p_i), x_{i+1}\in \B_{\beta' r_n}(p_{i+1})$,
  where $\beta'\leq \beta /2$, it holds that $x_i\preceq x_{i+1}$.
\end{corollary}

This motivates the construction of another set of balls.  For every
$1\leq i\leq M_n$ define the set
$B_{n,i}^\beta:=\B_{\beta r_n/2}(p_i)$.  Additionally, define
$K_n^\beta$ to be the number of $B_{n,i}^\beta$ balls which do not
contain samples from $\X_n$. Formally,
$K_n^\beta:=\left|\left\{B_{n,i}^\beta\cap
    \X_n=\emptyset\right\}_{i=1}^{M_n}\right|$.
We borrow the following lemma from~\cite[Lemma~4.3]{JanETAL15}.

\begin{lemma}\label{lem:pavone2}
  For any fixed $\alpha\in (0,1)$ it holds that $K_n^\beta<\alpha M_n$
  \as
\end{lemma}

\begin{corollary}
  For any fixed $\alpha\in (0,1)$ the following holds \as: $\L_n$
  contains a path $\nu_n$ connecting $\0$ to $\1$ such that
  $\M(\nu_n)\leq (1+\eps)\M(\nu^*)$, and with at most $2\alpha M_n$
  non-monotone edges.
\end{corollary}
\begin{proof}
  For every $1\leq i\leq M_n$ define $x_i\in \X_n$ to be some point in
  $\X_n\cap B_{n,i}^\beta$ if this set is non-empty, and otherwise
  some point in $B_{n,i}$. By Corollary~\ref{cor:quality}, the points
  $x_1,\ldots,x_{M_n}$ induce a path $\nu_n$ such that
  $\M(\nu_m)\leq (1+\eps)\M(\nu^*)$, and such a path exists \as By
  Corollary~\ref{cor:monotone}, every two consecutive points
  $x_i,x_{i+1}$ such that
  $x_i\in B_{n,i}^\beta, x_{i+1}\in B_{n,i+1}^\beta$ contribute one
  monotone edge to $\nu_n$. Conversely, each point
  $x_i\not\in B_{n,i}^\beta$ contributes at most two non-monotone
  edges. By Lemma~\ref{lem:pavone2}, there are at most $\alpha M_n$
  points of the latter type \as 
\end{proof}

It only remains to replace the two constants $\eps,\alpha$ with $o(1)$
in order to obtain the exact formulation of
Theorem~\ref{thm:main_new}. Since we have made no assumptions
concerning these two values, they can be replaced with the sequences
$\eps_i=1/i, \alpha_{i'}=1/i'$, where $i,i'\in \dN$. See more details
in~\cite[Theorem~6]{ssh-fne13}.

\section{Experiments}\label{sec:experiments}
\begin{figure*} \centering 
  \begin{subfigure}{0.28\textwidth}\centering
    \includegraphics[height=0.99\textwidth, angle=90,
    origin=c]{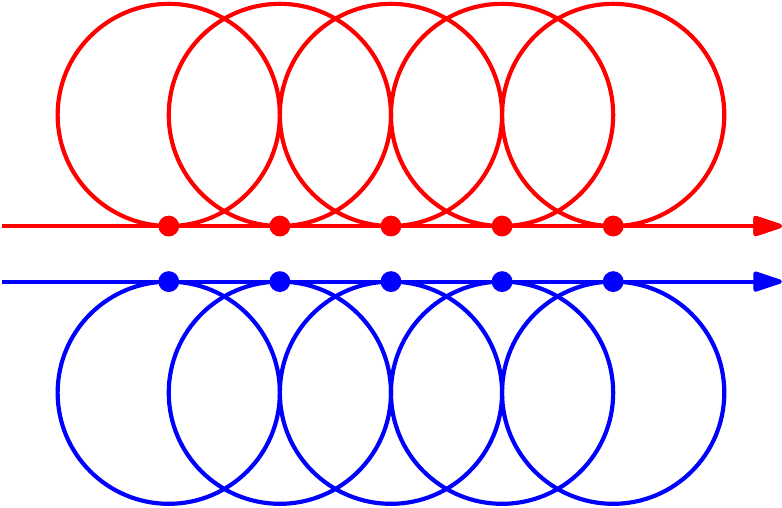}
   \caption{\label{fig:scenario_frechet}}
  \end{subfigure}
  ~
  \begin{subfigure}{0.5\textwidth}\centering
    \includegraphics[width=0.95\textwidth]{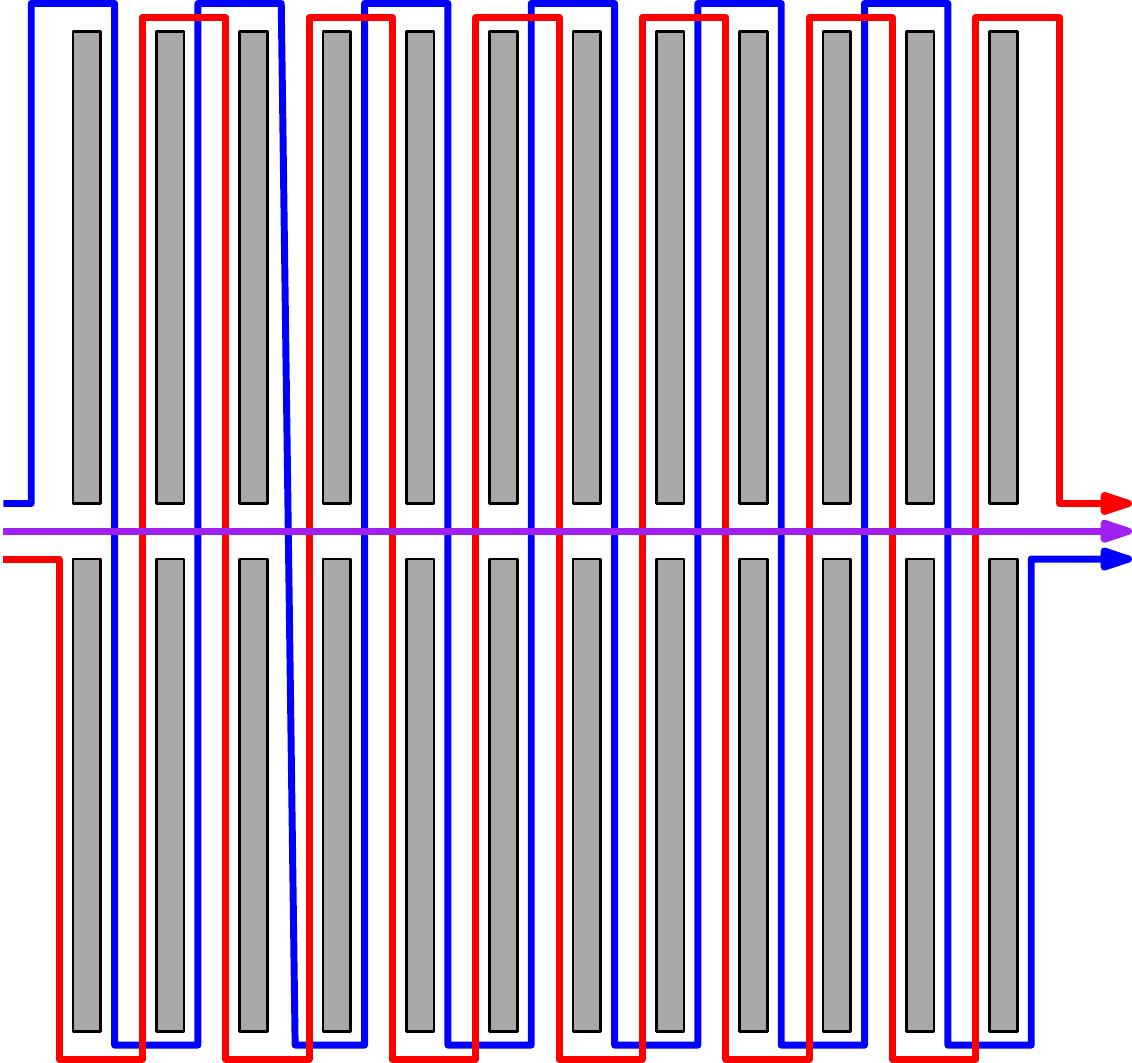}
    \caption{\label{fig:scenario_follow}}
  \end{subfigure}
  ~
  \begin{subfigure}{0.5\textwidth}\centering
    \includegraphics[height=0.95\textwidth]{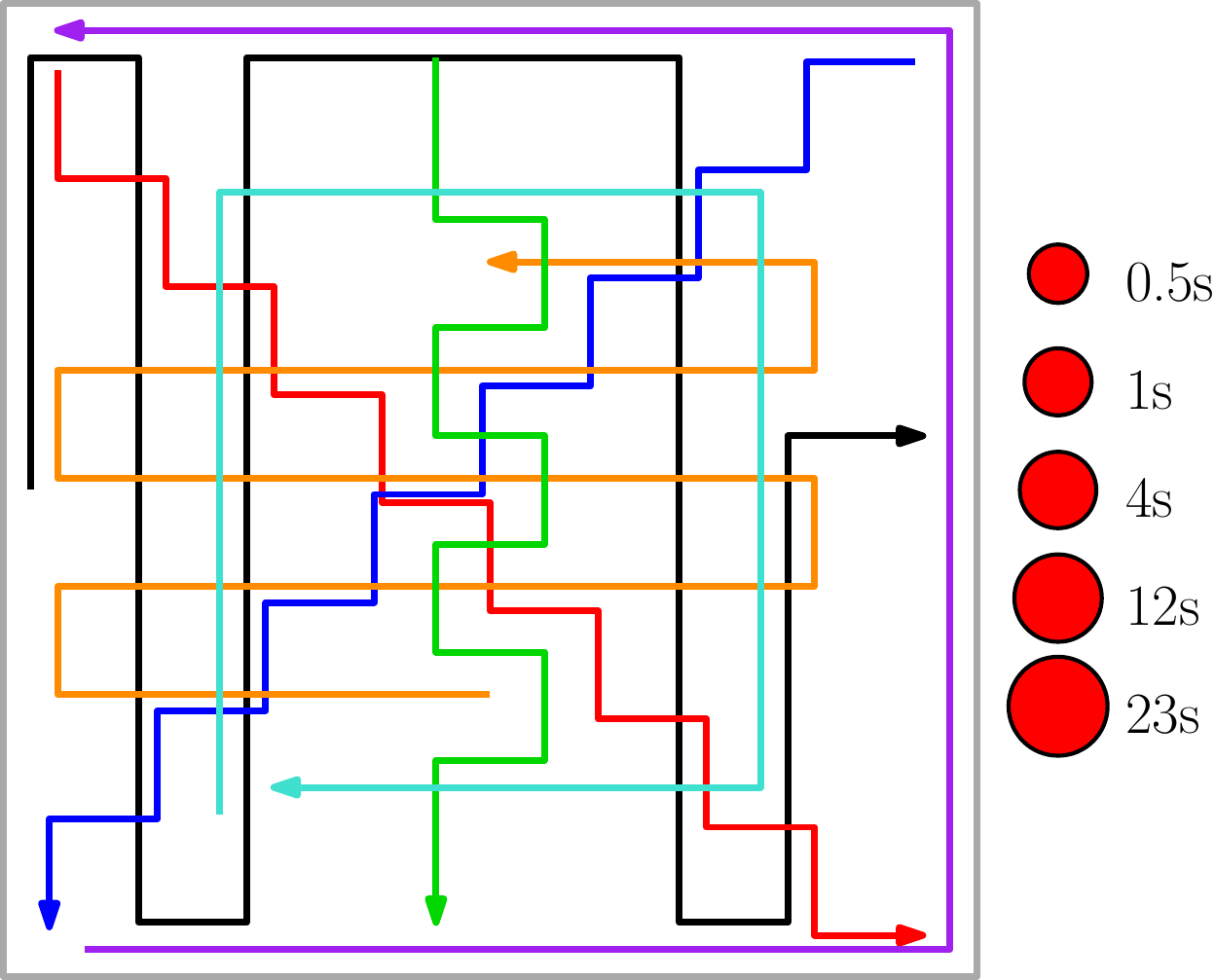}
    \caption{\label{fig:scenario_safety}}
  \end{subfigure}
  \caption{(a) Scenario for \textbf{(P1)}, for $d=2$: the right (blue)
    curve consists of five circular loops of radius $0.15$, where the
    entrance and exit point to each circle is indicated by a
    bullet. The left red curve is similarly defined, and the two
    curves are separated by a vertical distance of $0.04$. The optimal
    matching of cost $0.34$ is obtained in the following manner: when
    a given circle of the red curve is traversed, the position along
    the blue curve is fixed to the entrance point of the circle
    directly to the right of the traversed circle, and vice versa. (b)
    Scenario for \textbf{(P2)}: two agents moving along the red and
    blue curves, respectively, must minimize the distance to the
    leader on the horizontal purple curve while maintaining visibility
    with the leader. (c) Scenario for \textbf{(P3)}: finding the
    safest coordination between $d=7$ moving agents whose routes are
    depicted as red, blue, turquoise, orange, green, black, and purple
    curves. The red circles on the right represent the amount of
    separation obtained by \alg for the specified running time: given
    a plan returned by \alg, it induces a maximal radius of a disc,
    such that when $d$ copies of this disc are centered on the moving
    agents, they are guaranteed to remain collision-free.}
\end{figure*}

In this section we report on experimental results of applying \alg to
several challenging problems. We also compare its performance with
T-RRT*, which is the state-of-the-art for planning on cost maps. Since
BTT is concerned with optimality, we compare it against T-RRT* rather
than with its non-optimal original version T-RRT. Our technique is
faster than T-RRT* by several orders of magnitude. We mention that in
a previous work~\cite{SolHal16} we experimented on similar problems
but with a simpler PRM-flavor technique, which is significantly slower
than \alg. Due to this fact we do not compare against it. We also note
that the we compute fully monotone paths in all the experiments
reported below, using the formulation described in Alg.~\ref{alg:btt}.

We implemented \alg and T-RRT* in \Cpp, and tested them on scenarios
involving two-dimensional objects.  \alg has only two parameters,
which are the number of samples $n$ and the connection radius
$r_n$. The latter was set to the value described in
Theorem~\ref{thm:main_new} with $\eta=1$. We note that even though
this theorem only suggests that this value suffices for asymptotic
optimality in the monotone case, the experiments support this claim
also after finite, relatively short, running time. In particular, \alg
obtains an initial solution fairly quickly and converges to the
optimum, in scenarios where the optimum is known.

Our implementation of T-RRT* is based on its
\textsc{ompl}~\cite{OMPL12} implementation. T-RRT* uses a connection
radius, which we set according to~\cite{KarFra11}.  T-RRT* also has
the parameters $T,T_{\text{rate}}$, which are set according to the
guidelines in~\cite{DevETAL16}. The length of the \texttt{extend}
routine, and the rate of target bias, were set according to the
\textsc{ompl} implementation of T-RRT*.

Nearest-neighbor search calls in T-RRT* were performed using
\textsc{flann}~\cite{flann}. As the set of samples of \alg is
generated in one batch, we use the efficient RTG data structure (see,
e.g.,~\cite{KleETAL15}) for nearest-neighbor search. Geometric
objects
were represented with \textsc{cgal}~\cite{cgal}. Experiments were
conducted on a PC with Intel i7-2600 3.4GHz processor with 8GB of
memory, running a 64-bit Windows~7 OS.

\subsection{Scenarios}
We test \alg on several challenging problems, which are described
below. We mention that the underlying implementation of \alg and
T-RRT* does not change between the different scenarios and we only
change the subroutine responsible for the computation of cost-map
values. \vspace{3pt}

\noindent\textbf{(P1) \frechet matching:}  The goal is to
find a traversal which \emph{minimizes} the pairwise distance between
the traversal points along $d\in\{2,3,4\}$ curves. Recall that given
$d\geq 2$ curves $\sigma_1,\ldots,\sigma_d:[0,1]\rightarrow \dR^2$,
the cost map $\M:[0,1]^d\rightarrow \dR_{+}$ induced by the problem of
\frechet matching~\cite{AltGod95} is defined to be
$\M(\tau_1,\ldots,\tau_d)=\max_{1\leq i<j\leq
  d}\|\sigma_i(\tau_i)-\sigma_j(\tau_j)\|$
for $\tau_1,\ldots,\tau_d\in [0,1]$.  The scenario for $d=2$ is
depicted in Fig.~\ref{fig:scenario_frechet}. For~$d=3$ we add an
identical blue curve, and for $d=4$ another red curve. Note that the
optimum is the same for all values of $d$ here. \vspace{3pt}

\noindent\textbf{(P2) Leader following:} The problem consists of three
curves, where the purple one (Fig.~\ref{fig:scenario_follow})
represents the motion of a ``leader'', whereas the blue and the red
curves represent the ``followers''. The goal is to plan the motion of
the three agents such that at least one of the followers sees the
(current position of the) leader at any given time, where the view of
each agent can be obstructed by the gray walls. In addition, we must
minimize the distance between the leader to its closest follower at a
given time. The optimal solution is
attained in the following manner: while the leader passes between two
vertically-opposite walls the blue follower keeps an eye on it while
the red follower goes around a similar set of walls. Then, when the
leader becomes visible to the red agent, the blue agent catches up
with the other two agents in an analogous fashion, and this process is
repeated.\vspace{3pt}

\noindent\textbf{(P3) Safest coordination:} This problem consists of finding the
\emph{safest coordination} among $d=7$ agents moving along predefined
routes, which represents a complex traffic intersection. The goal is
to find the traversal which \emph{maximizes} the pairwise distance
between the $d$ traversal points. This can be viewed as
``anti''-\frechet coordination. Each route is drawn in a different
color, where the direction is indicated by an arrow (see
Fig.~\ref{fig:scenario_safety}).\vspace{3pt}

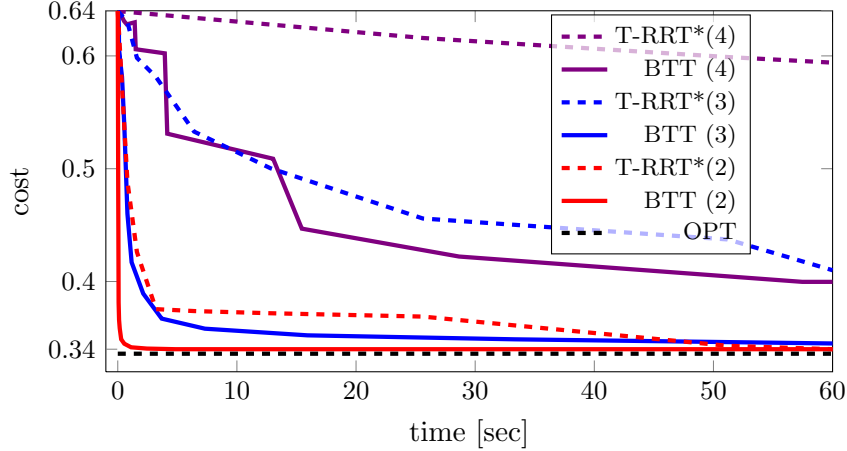
\begin{figure}\centering 
    \begin{tikzpicture}
    \begin{axis}[
      height=0.4 \columnwidth, width=0.7\columnwidth,
      xmin=-1, xmax=60, ymin=0.32, ymax=0.64,
      extra y ticks={0.34,0.64},
      xlabel={time [sec]},
      ylabel={cost},
      legend style={at={(0.75,0.99)}, anchor=north, cells={anchor=east}},fill=white, fill opacity=0.7, draw opacity=1,text opacity=1 ]
            ]
      \pgfplotsset{
        tick label style={font=\small},
        label style={font=\normalsize},
        legend style={font=\footnotesize}
      }
      
    \addplot[dashed,violet, ultra thick] plot coordinates {
      (0.00,0.63972)
      (25.6154,0.61588)
      (51.212,0.59898)
      (60.4116,0.59379)
    }; 
    \addlegendentry{T-RRT*(4)}
    
    \addplot[solid,violet, ultra thick] plot coordinates {
      (0.1714,0.63848)
      (0.2934,0.63789)
      (0.3434,0.63465)
      (0.5306,0.63082)
      (0.7554,0.62797)
      (1.421,0.62976)
      (1.5068,0.6059)
      (3.9446,0.60228)
      (4.1546,0.53105)
      (13.0416,0.50883)
      (15.4486,0.44687)
      (28.645,0.42215)
      (57.457,0.3997)
      (60.0,0.3997)
    }; 
    \addlegendentry{\alg(4)}

    \addplot[dashed,blue, ultra thick] plot coordinates {
      (0.00,0.63972)
      (0.8082,0.63001)
      (1.6094,0.59827)
      (3.2134,0.58123)
      (6.4072,0.5329)
      (12.8314,0.50034)
      (25.6088,0.45586)
      (51.2068,0.43729)
      (60,0.41)
    }; 
    \addlegendentry{T-RRT*(3)}
    
    \addplot[solid,blue, ultra thick] plot coordinates {
     (0.0622,0.63972)
     (0.0934,0.62123)
     (0.1188,0.58934)
     (0.1872,0.59677)
     (0.3558,0.58147)
     (0.5514,0.53759)
     (0.8108,0.46019)
     (1.1632,0.41701)
     (2.1372,0.38929)
     (3.6884,0.36711)
     (7.3136,0.35819)
     (15.9178,0.35229)
     (33.338,0.34876)
     (60,0.34508)
    }; 
    \addlegendentry{\alg(3)}

    \addplot[dashed,red, ultra thick] plot coordinates {
      (0.00,0.63972)
      (0.811,0.48841)
      (1.608,0.42614)
      (3.2082,0.37551)
      (6.4098,0.37381)
      (12.806,0.37188)
      (25.6104,0.36891)
      (51.201, 0.343)
      (60.0, 0.34)
    }; 
    \addlegendentry{T-RRT*(2)}
    
    \addplot[solid,red, ultra thick] plot coordinates {
      (0.00,0.63972)
      (0.037,0.48472)
      (0.0502,0.44821)
      (0.0716,0.38132)
      (0.1092,0.36596)
      (0.1656,0.35845)
      (0.2868,0.34877)
      (0.5494,0.34485)
      (1.118,0.34169)
      (2.4416,0.34056)
      (4.857,0.34)
      (60.0,0.34)
    }; 
    \addlegendentry{\alg(2)}

    \addplot[dashed,black, ultra thick] plot coordinates {
      (0.0,0.336)
      (60.0,0.336)
    };
    \addlegendentry{OPT}
  \end{axis}
\end{tikzpicture}
\caption{Results of \alg and T-RRT* on \textbf{(P1)} for increasing number
  of curves $2\leq d\leq 4$, indicated in parenthesis. The black
  dashed line represents the optimal cost $0.34$. In order to display
  the quality of the solution of \alg as a function of the running
  time, rather than the number of samples, we ran \alg on an
  increasing number of samples, with ten runs for each number. For
  each $n$ we plot the average cost and the average running time. We
  note that the standard deviation of the running time is very low,
  and so the average accurately captures the typical running time. For
  instance, for $d=3$ the standard deviation ranges between 0.002 and
  0.01. \label{plot:frechet}}
\end{figure}

\subsection{Results}
We report on the running times and the obtained cost values of \alg
and T-RRT*, after averaging over ten runs in each setting. 

We start with \textbf{(P1)}. For $d\in \{2,3\}$, \alg rapidly
converges to the optimum (see Fig.~\ref{plot:frechet}). For $d=4$ the
convergence is slower, although a value relatively close to the
optimum is reached (observe that the trivial solution in which the
curves proceed simultaneously induces a cost of $0.64$). Note that
\alg's performance for $d=3$ is better than that of T-RRT* for
$d=2$. A similar phenomenon occurs for \alg with $d=4$ and T-RRT* for
$d=3$. Lastly, for $d=4$ T-RRT* manages to find a solution that is
only slightly better than the trivial one after~60s. 

In~\textbf{(P2)} \alg attains the (near-)optimal solution described
earlier after 13s on average. Here T-RRT* is unable to obtain any
solution within 60s.  In \textbf{(P3)}, again, T-RRT* is unable to
obtain a solution within 60s. In contrast, \alg obtains a solution in
less than a second. Observe that the solution obtained by it for 22
seconds maintains a rather large safety distance between the agents in
such a complex setting (see caption of
Fig.~\ref{fig:scenario_safety}).  Both in \textbf{(P2)} and
\textbf{(P3)} T-RRT*'s transition test leads to an extremely slow
exploration of the parameter space. We mention that in these cases
RRT* is able to find an initial solution quite rapidly, albeit one of
poor quality. Moreover, the convergence rate of RRT* is very slow and
hardly manages to improve over the solution obtained initially.

Now we analyze in greater detail the running time of \alg as a
function of the number of samples $n$ for the specific scenario in
\textbf{(P3)}.  The running time is dominated by three components:
(i)~construction of the RTG NN-search data structure; (ii)~NN-search
calls; (iii)~computation of the cost map, which corresponds to
collision detection in standard motion planning. \alg's
proportion of running time spent between these components changes with
$n$ (see Fig.~\ref{plot:percentage}). In particular, as $n$ increases
so does the proportion of~NN calls. This trend can be observed in
other planners as well~\cite{KleETAL15}.

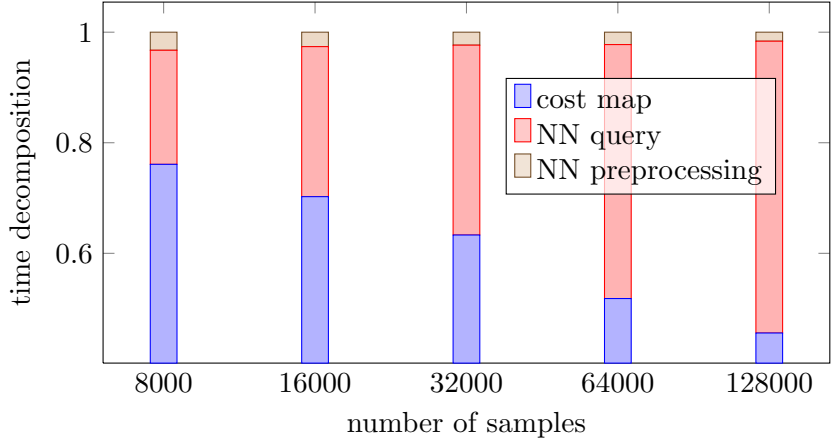
\begin{figure}\centering 
\begin{tikzpicture}
\begin{axis}[
  height=0.4\columnwidth, width=0.7\columnwidth,
    ybar stacked,
    legend style={at={(0.74,0.79)},
      anchor=north, cells={anchor=west}, fill=white, fill opacity=0.7, draw opacity=1,text opacity=1},
    ylabel={time decomposition},
    xlabel={number of samples},
    symbolic x coords={8000,16000,32000,64000,128000},
    xtick=data,
    x tick label style={/pgf/number format/1000 sep=},
    ]
    \addplot+[ybar] plot coordinates {(8000,0.760959) (16000,0.702479)
      (32000,0.633087) (64000,0.518196) (128000,0.455842)};
    \addplot+[ybar] plot coordinates {(8000,0.2063414)
      (16000,0.2714062) (32000,0.3434418) (64000,0.4593758)
      (128000,0.5278556)}; 
    \addplot+[ybar] plot coordinates
    {(8000,0.03269988) (16000,0.02611522) (32000,0.02347034)
      (64000,0.02242788) (128000,0.01630252)}; \legend{cost map, NN
      query, NN preprocessing}
\end{axis}
\end{tikzpicture}
\caption{Time percentage for each of the main components of \alg for
  \textbf{(P3)}.\label{plot:percentage}}
\end{figure}

\section{Discussion and future work}\label{sec:discussion}
We have introduced \alg for sampling-based bottleneck pathfinding over
cost maps. We showed that it manages to cope with complex scenarios of
the problem and outperforms T-RRT* in all tests. In addition to its
efficiency, \alg requires the tuning of only a single parameter: while
the number of samples~$n$ cannot be determined a priori, the
connection radius $r_n$ specified in Theorem~\ref{thm:main_new} proves
to be sufficient for all tested cases. We also provide theoretical
justification for this phenomenon in the same theorem.  Note that this
connection radius is widely used in non-monotone settings of motion
planning (see \textsc{ompl}~\cite{OMPL12}). Thus, we believe that
$r_n$ can be fixed to the aforementioned value.  In contrast, T-RRT*
requires the tuning of two more parameters that are unique to that
planner, as well as those that are inherited from RRT and RRT*.

We do note that \alg might benefit from biasing some of its search
towards the target as is the case in RRT-type planners. We intend to
pursue this direction in future work. We also mention that since \alg
is based on Dijkstra's shortest-path algorithm, it may be applicable
to other cost functions such as integral cost and path length, by
modifying the edge cost function (lines~12,13) and the underlying
discrete path planner (line~8) in Algorithm~\ref{alg:btt}.

Currently the main obstacles to using \alg in more complex settings of
bottleneck pathfinding, e.g., higher dimensions, is the high memory
consumption of the RTG data structure. In particular, for $n=10^{6}$
and $d=7$ the program may exceed the 4\textsc{gb} limit provided by
the OS.  Interestingly, in all the tested cases \alg examines only a
small portion of the actual vertex set of $\G_n$ (at most
$5\%$). Moreover, the examined vertices are usually grouped together in
contiguous regions of the parameter space. This calls for incremental
versions of \alg and RTG which generate samples and preprocess them in
an online fashion.

On the theoretical side, we aim to refine the statement made in
Theorem~\ref{thm:main_new}. Currently we can only ensure that a
solution that is ``almost monotone'' exists, but does a
``fully-monotone'' solution over $\G_n$ exists as well? We conjecture
that this statement is true.

\bibliography{bibliography}

\begin{thebibliography}{51}
\providecommand{\natexlab}[1]{#1}
\providecommand{\url}[1]{\texttt{#1}}
\expandafter\ifx\csname urlstyle\endcsname\relax
  \providecommand{\doi}[1]{doi: #1}\else
  \providecommand{\doi}{doi: \begingroup \urlstyle{rm}\Url}\fi

\bibitem[Adler et~al.(2015)Adler, de~Berg, Halperin, and
  Solovey]{abhs-unlabeled14}
A.~Adler, M.~de~Berg, D.~Halperin, and K.~Solovey.
\newblock Efficient multi-robot motion planning for unlabeled discs in simple
  polygons.
\newblock \emph{{T-ASE}}, 12\penalty0 (4):\penalty0 1309--1317, 2015.

\bibitem[Agarwal et~al.(2016)Agarwal, Fox, and Salzman]{AgaETAL16}
P.~K. Agarwal, K.~Fox, and O.~Salzman.
\newblock An efficient algorithm for computing high-quality paths amid
  polygonal obstacles.
\newblock In \emph{{SODA}}, pages 1179--1192, 2016.

\bibitem[Alt and Godau(1995)]{AltGod95}
H.~Alt and M.~Godau.
\newblock Computing the {F}r{\'{e}}chet distance between two polygonal curves.
\newblock \emph{Int. J. Comput. Geometry Appl.}, 5:\penalty0 75--91, 1995.

\bibitem[Altch{\'{e}} et~al.(2016)Altch{\'{e}}, Qian, and
  de~La~Fortelle]{AltETAL16}
F.~Altch{\'{e}}, X.~Qian, and A.~de~La~Fortelle.
\newblock Time-optimal coordination of mobile robots along specified paths.
\newblock \emph{CoRR}, abs/1603.04610, 2016.

\bibitem[Arslan and Tsiotras(2013)]{ArsTsi13}
O.~Arslan and P.~Tsiotras.
\newblock Use of relaxation methods in sampling-based algorithms for optimal
  motion planning.
\newblock In \emph{{ICRA}}, pages 2421--2428, 2013.

\bibitem[Buchin et~al.(2008)Buchin, Buchin, and Wenk]{BucBucWen08}
K.~Buchin, M.~Buchin, and C.~Wenk.
\newblock Computing the {F}r{\'{e}}chet distance between simple polygons.
\newblock \emph{Comput. Geom.}, 41\penalty0 (1-2):\penalty0 2--20, 2008.

\bibitem[Buchin et~al.(2010)Buchin, Buchin, and Schulz]{BucBucSch08}
K.~Buchin, M.~Buchin, and A.~Schulz.
\newblock {F}r{\'{e}}chet distance of surfaces: Some simple hard cases.
\newblock In \emph{{ESA}}, pages 63--74, 2010.

\bibitem[Buchin et~al.(2014)Buchin, Buchin, Meulemans, and
  Mulzer]{BucBucMeuMul14}
K.~Buchin, M.~Buchin, W.~Meulemans, and W.~Mulzer.
\newblock Four soviets walk the dog - with an application to {A}lt's
  conjecture.
\newblock In \emph{{SODA}}, pages 1399--1413, 2014.

\bibitem[Buchin et~al.(2016)Buchin, Buchin, Konzack, Mulzer, and
  Schulz]{BucETAL16}
K.~Buchin, M.~Buchin, M.~Konzack, W.~Mulzer, and A.~Schulz.
\newblock Fine-grained analysis of problems on curves.
\newblock In \emph{{EuroCG}}, 2016.

\bibitem[Cafieri and Durand(2014)]{CafDur14}
S.~Cafieri and N.~Durand.
\newblock Aircraft deconfliction with speed regulation: new models from
  mixed-integer optimization.
\newblock \emph{Journal of Global Optimization}, 58\penalty0 (4):\penalty0
  613--629, 2014.

\bibitem[Chambers et~al.(2010)Chambers, de~Verdi{\`{e}}re, Erickson, Lazard,
  Lazarus, and Thite]{ChaETAL10}
E.~W. Chambers, {\'{E}}.~C. de~Verdi{\`{e}}re, J.~Erickson, S.~Lazard,
  F.~Lazarus, and S.~Thite.
\newblock Homotopic {F}r{\'{e}}chet distance between curves or, walking your
  dog in the woods in polynomial time.
\newblock \emph{Comput. Geom.}, 43\penalty0 (3):\penalty0 295--311, 2010.

\bibitem[Cook and Wenk(2010)]{CooWen10}
A.~F. Cook and C.~Wenk.
\newblock Geodesic {F}r{\'{e}}chet distance inside a simple polygon.
\newblock \emph{{ACM} Transactions on Algorithms}, 7\penalty0 (1):\penalty0 9,
  2010.

\bibitem[Cormen et~al.(2009)Cormen, Leiserson, Rivest, and Stein]{CorETAL09}
T.~H. Cormen, C.~E. Leiserson, R.~L. Rivest, and C.~Stein.
\newblock \emph{Introduction to Algorithms {(3.} ed.)}.
\newblock {MIT} Press, 2009.

\bibitem[Dasler and Mount(2015)]{DasMou15}
P.~Dasler and D.~M. Mount.
\newblock On the complexity of an unregulated traffic crossing.
\newblock In \emph{WADS}, pages 224--235, 2015.

\bibitem[de~Berg and van Kreveld(1997)]{BerKre97}
M.~de~Berg and M.~J. van Kreveld.
\newblock Trekking in the alps without freezing or getting tired.
\newblock \emph{Algorithmica}, 18\penalty0 (3):\penalty0 306--323, 1997.

\bibitem[Devaurs et~al.(2016)Devaurs, Sim{\'{e}}on, and
  Cort{\'{e}}s]{DevETAL16}
D.~Devaurs, T.~Sim{\'{e}}on, and J.~Cort{\'{e}}s.
\newblock Optimal path planning in complex cost spaces with sampling-based
  algorithms.
\newblock \emph{{T-ASE}}, 13\penalty0 (2):\penalty0 415--424, 2016.

\bibitem[Geraerts and Overmars(2007)]{GerOve07}
R.~Geraerts and M.~Overmars.
\newblock Creating high-quality paths for motion planning.
\newblock \emph{I.~J. Robotics Res.}, 26\penalty0 (8):\penalty0 845--863, 2007.

\bibitem[Har{-}Peled and Raichel(2014)]{HarRai14}
S.~Har{-}Peled and B.~Raichel.
\newblock The {F}r{\'{e}}chet distance revisited and extended.
\newblock \emph{{ACM} Transactions on Algorithms}, 10\penalty0 (1):\penalty0 3,
  2014.

\bibitem[Holladay and Srinivasa(2016)]{HolSri16}
R.~Holladay and S.~Srinivasa.
\newblock Distance metrics and algorithms for task space path optimization.
\newblock In \emph{{IROS}}, 2016.
\newblock to appear.

\bibitem[Hopcroft et~al.(1984)Hopcroft, Schwartz, and Sharir]{hss-cmpmio}
J.~E. Hopcroft, J.~T. Schwartz, and M.~Sharir.
\newblock On the complexity of motion planning for multiple independent
  objects; {PSPACE}-hardness of the ``{W}arehouseman's problem''.
\newblock \emph{I. J. Robotic Res.}, 3\penalty0 (4):\penalty0 76--88, 1984.

\bibitem[Jaillet et~al.(2010)Jaillet, Cort{\'{e}}s, and
  Sim{\'{e}}on]{JaiETAK10}
L.~Jaillet, J.~Cort{\'{e}}s, and T.~Sim{\'{e}}on.
\newblock Sampling-based path planning on configuration-space costmaps.
\newblock \emph{Trans. Robotics}, 26\penalty0 (4):\penalty0 635--646, 2010.

\bibitem[Janson et~al.(2015)Janson, Schmerling, Clark, and Pavone]{JanETAL15}
L.~Janson, E.~Schmerling, A.~A. Clark, and M.~Pavone.
\newblock Fast marching tree: {A} fast marching sampling-based method for
  optimal motion planning in many dimensions.
\newblock \emph{I. J. Robotic Res.}, 34\penalty0 (7):\penalty0 883--921, 2015.

\bibitem[Karaman and Frazzoli(2011)]{KarFra11}
S.~Karaman and E.~Frazzoli.
\newblock Sampling-based algorithms for optimal motion planning.
\newblock \emph{I. J. Robotic Res.}, 30\penalty0 (7):\penalty0 846--894, 2011.

\bibitem[Kavraki et~al.(1996{\natexlab{a}})Kavraki, Kolountzakis, and
  Latombe]{KavETAL96}
L.~E. Kavraki, M.~N. Kolountzakis, and J.~Latombe.
\newblock Analysis of probabilistic roadmaps for path planning.
\newblock In \emph{{ICRA}}, pages 3020--3025, 1996{\natexlab{a}}.

\bibitem[Kavraki et~al.(1996{\natexlab{b}})Kavraki, \v{S}vestka, Latombe, and
  Overmars]{kslo-prm}
L.~E. Kavraki, P.~\v{S}vestka, J.-C. Latombe, and M.~H. Overmars.
\newblock Probabilistic roadmaps for path planning in high dimensional
  configuration spaces.
\newblock \emph{Trans. Robotics}, 12\penalty0 (4):\penalty0 566--580,
  1996{\natexlab{b}}.

\bibitem[Kleinbort et~al.(2015)Kleinbort, Salzman, and Halperin]{KleETAL15}
M.~Kleinbort, O.~Salzman, and D.~Halperin.
\newblock Efficient high-quality motion planning by fast all-pairs
  r-nearest-neighbors.
\newblock In \emph{{ICRA}}, pages 2985--2990, 2015.

\bibitem[Kuffner and LaValle(2000)]{l-rert}
J.~J. Kuffner and S.~M. LaValle.
\newblock {RRT}-{C}onnect: An efficient approach to single-query path planning.
\newblock In \emph{ICRA}, pages 995--1001, 2000.

\bibitem[Mitchell(2004)]{Mitchell04}
J.~S.~B. Mitchell.
\newblock Shortest paths and networks.
\newblock In \emph{Handbook of Discrete and Computational Geometry, Second
  Edition.}, pages 607--641. 2004.

\bibitem[Muja and Lowe(2009)]{flann}
M.~Muja and D.~G. Lowe.
\newblock {Fast approximate nearest neighbors with automatic algorithm
  configuration}.
\newblock In \emph{VISSAPP}, pages 331--340. INSTICC Press, 2009.

\bibitem[Nechushtan et~al.(2010)Nechushtan, Raveh, and Halperin]{NecETAL10}
O.~Nechushtan, B.~Raveh, and D.~Halperin.
\newblock Sampling-diagram automata: {A} tool for analyzing path quality in
  tree planners.
\newblock In \emph{WAFR}, pages 285--301, 2010.

\bibitem[{\'{O}}'D{\'{u}}nlaing and Yap(1985)]{ODuYap85}
C.~{\'{O}}'D{\'{u}}nlaing and C.~Yap.
\newblock A "retraction" method for planning the motion of a disc.
\newblock \emph{J. Algorithms}, 6\penalty0 (1):\penalty0 104--111, 1985.

\bibitem[Pokorny et~al.(2016)Pokorny, Goldberg, and Kragic]{PokETAL16}
F.~Pokorny, K.~Goldberg, and D.~Kragic.
\newblock Topological trajectory clustering with relative persistent homology.
\newblock In \emph{{ICRA}}, pages 16--23, 2016.

\bibitem[Raveh et~al.(2011)Raveh, Enosh, and Halperin]{RavETAL11}
B.~Raveh, A.~Enosh, and D.~Halperin.
\newblock A little more, a lot better: Improving path quality by a path-merging
  algorithm.
\newblock \emph{Trans. Robotics}, 27\penalty0 (2):\penalty0 365--371, 2011.

\bibitem[Reif(1979)]{Rei79}
J.~H. Reif.
\newblock Complexity of the mover’s problem and generalizations: {E}xtended
  abstract.
\newblock In \emph{{FOCS}}, pages 421--427, 1979.

\bibitem[Salzman and Halperin(2016)]{SalHal14a}
O.~Salzman and D.~Halperin.
\newblock Asymptotically near-optimal {RRT} for fast, high-quality motion
  planning.
\newblock \emph{Trans. Robotics}, 32\penalty0 (3):\penalty0 473--483, 2016.

\bibitem[Sharir(2004)]{Sharir04}
M.~Sharir.
\newblock Algorithmic motion planning.
\newblock In J.~E. Goodman and J.~O'Rourke, editors, \emph{Handbook of Discrete
  and Computational Geometry, Second Edition.}, pages 1037--1064. Chapman and
  Hall/CRC, 2004.

\bibitem[Simeon et~al.(2002)Simeon, Leroy, and Laumond]{SimETAL02}
T.~Simeon, S.~Leroy, and J.~P. Laumond.
\newblock Path coordination for multiple mobile robots: a resolution-complete
  algorithm.
\newblock \emph{Trans. Robotics}, 18\penalty0 (1):\penalty0 42--49, 2002.

\bibitem[Solovey and Halperin(2015)]{SolHal15}
K.~Solovey and D.~Halperin.
\newblock On the hardness of unlabeled multi-robot motion planning.
\newblock In \emph{{RSS}}, 2015.

\bibitem[Solovey and Halperin(2016)]{SolHal16}
K.~Solovey and D.~Halperin.
\newblock Sampling-based bottleneck pathfinding with applications to
  {F}r\'echet matching.
\newblock In \emph{{ESA}}, pages 76:1--76:16, 2016.

\bibitem[Solovey et~al.(2015)Solovey, Yu, Zamir, and Halperin]{SolYuZamHal15}
K.~Solovey, J.~Yu, O.~Zamir, and D.~Halperin.
\newblock Motion planning for unlabeled discs with optimality guarantees.
\newblock In \emph{{RSS}}, 2015.

\bibitem[Solovey et~al.(2016{\natexlab{a}})Solovey, Salzman, and
  Halperin]{SolETAL16}
K.~Solovey, O.~Salzman, and D.~Halperin.
\newblock New perspective on sampling-based motion planning via random
  geometric graphs.
\newblock In \emph{{RSS}}, 2016{\natexlab{a}}.

\bibitem[Solovey et~al.(2016{\natexlab{b}})Solovey, Salzman, and
  Halperin]{ssh-fne13}
K.~Solovey, O.~Salzman, and D.~Halperin.
\newblock Finding a needle in an exponential haystack: {D}iscrete {RRT} for
  exploration of implicit roadmaps in multi-robot motion planning.
\newblock \emph{I. J. Robotic Res.}, 35\penalty0 (5):\penalty0 501--513,
  2016{\natexlab{b}}.

\bibitem[Spensieri et~al.(2013)Spensieri, Bohlin, and Carlson]{SpeETAL13}
D.~Spensieri, R.~Bohlin, and J.~S. Carlson.
\newblock Coordination of robot paths for cycle time minimization.
\newblock In \emph{{CASE}}, pages 522--527, 2013.

\bibitem[Spirakis and Yap(1984)]{sy-snp84}
P.~G. Spirakis and C.-K. Yap.
\newblock Strong {NP}-hardness of moving many discs.
\newblock \emph{Information Processing Letters}, 19\penalty0 (1):\penalty0
  55--59, 1984.

\bibitem[{\c{S}}ucan et~al.(2012){\c{S}}ucan, Moll, and Kavraki]{OMPL12}
I.~A. {\c{S}}ucan, M.~Moll, and L.~E. Kavraki.
\newblock The {O}pen {M}otion {P}lanning {L}ibrary.
\newblock \emph{{IEEE} Robotics \& Automation}, 19\penalty0 (4):\penalty0
  72--82, 2012.
\newblock \url{http://ompl.kavrakilab.org}.

\bibitem[{The CGAL Project}(2016)]{cgal}
{The CGAL Project}.
\newblock \emph{{CGAL} user and reference manual}.
\newblock {CGAL editorial board}, {4.8} edition, 2016.
\newblock URL \url{http://www.cgal.org/}.

\bibitem[Turpin et~al.(2013)Turpin, Michael, and Kumar]{tmk-cap13}
M.~Turpin, N.~Michael, and V.~Kumar.
\newblock Concurrent assignment and planning of trajectories for large teams of
  interchangeable robots.
\newblock In \emph{{ICRA}}, pages 842--848, 2013.

\bibitem[van~den Berg and Overmars(2005)]{vanBerOve05}
J.~P. van~den Berg and M.~H. Overmars.
\newblock Prioritized motion planning for multiple robots.
\newblock In \emph{{IROS}}, pages 430--435, 2005.

\bibitem[Voss et~al.(2015)Voss, Moll, and Kavraki]{VosETAL15}
C.~Voss, M.~Moll, and L.~E. Kavraki.
\newblock A heuristic approach to finding diverse short paths.
\newblock In \emph{{ICRA}}, pages 4173--4179, 2015.

\bibitem[Wein et~al.(2007)Wein, van~den Berg, and Halperin]{WeiETAL07}
R.~Wein, J.~P. van~den Berg, and D.~Halperin.
\newblock The visibility-voronoi complex and its applications.
\newblock \emph{Comput. Geom.}, 36\penalty0 (1):\penalty0 66--87, 2007.

\bibitem[Wein et~al.(2008)Wein, van~den Berg, and Halperin]{WeiETAL08}
R.~Wein, J.~P. van~den Berg, and D.~Halperin.
\newblock Planning high-quality paths and corridors amidst obstacles.
\newblock \emph{I. J. Robotic Res.}, 27\penalty0 (11-12):\penalty0 1213--1231,
  2008.

\end{thebibliography}

\end{document}